\documentclass{article}

\usepackage{times}
\usepackage{graphicx} 
\usepackage{subfigure} 

\usepackage{natbib}

\usepackage{algorithm}
\usepackage{algorithmic}

\usepackage[draft]{hyperref}

\usepackage[accepted]{icml2016}

\usepackage{amsthm}
\usepackage{tabularx}
\usepackage[space]{grffile}
\usepackage{amsmath}
\usepackage{amssymb}
\usepackage{bbm}
\usepackage{amsfonts}
\usepackage{array}
\usepackage{multirow}
\usepackage{graphicx}
\usepackage{mathtools}
\newcommand\MyBox[2]{
  \fbox{\lower0.75cm
    \vbox to 1.7cm{\vfil
      \hbox to 1.7cm{\hfil\parbox{1.4cm}{#1\\#2}\hfil}
      \vfil}%
  }%
}

\allowdisplaybreaks
\usepackage[utf8]{inputenc}
\usepackage[T1]{fontenc}
\newtheorem{theorem}{Theorem}
\usepackage{cprotect}

\icmltitlerunning{Clustering on the Edge: Learning Structure in Graphs}

\begin{document} 

\twocolumn[
\icmltitle{Clustering on the Edge: Learning Structure in Graphs}

\icmlauthor{Matt Barnes}{mbarnes1@cs.cmu.edu}
\icmladdress{Carnegie Mellon University,
            5000 Forbes Avenue, Pittsburgh, PA 15213 USA}
\icmlauthor{Artur Dubrawski}{awd@cs.cmu.edu}
\icmladdress{Carnegie Mellon University,
            5000 Forbes Avenue, Pittsburgh, PA 15213 USA}

\icmlkeywords{Planted partition model, Stochastic block model, Correlation clustering}

\vskip 0.3in
]

\begin{abstract} 
With the recent popularity of graphical clustering methods, there has been an increased focus on the information \emph{between} samples. We show how learning cluster structure using edge features naturally and simultaneously determines the most likely number of clusters and addresses data scale issues. These results are particularly useful in instances where (a) there are a large number of clusters and (b) we have some labeled edges. Applications in this domain include image segmentation, community discovery and entity resolution. Our model is an extension of the planted partition model and our solution uses results of correlation clustering, which achieves a partition $\mathcal{O}(log(n))$-close to the log-likelihood of the true clustering.
\end{abstract} 

\section{Introduction}
Graphical approaches to clustering are appealing because they offer a natural way to compare samples, in the form of edge information.
Broadly, a good clustering should have similar nodes in the same cluster and dissimilar nodes in different clusters. However, which graph to use for clustering remains an open question \cite{Luxburg2007}. Previous work has considered edges to be the output of a similarity function\footnote{A similarity function is a function of two nodes' features, e.g. the RBF kernel} (e.g. spectral clustering), a Bernoulli random variable (e.g. stochastic block models), or some more general measure of similarity/dissimilarity (e.g. correlation clustering).

In reality, edge information can take a variety and multiplicity of forms. Edges in social graphs correspond to communication exchanges, mutual interests and types of relationships. In biology, protein-protein interaction networks involve complex underlying mechanisms and conditions under which an event may occur. And in economics, trades constitute numerous goods, prices and transaction types.

We are inspired by the complex interactions happening around us. Our relationships are more complicated than friend/not-friend, and our transactions are about more than the monetary value.
The motivation of this paper is to cluster with the additional information provided by multivariate edge features. This is partly supported by Thomas and Blitzstein's \yrcite{Thomas2011} recent results showing that converting to a binary graph makes recovering a partition more difficult. We are also interested in how to choose similarity functions which better capture the relationship between nodes, one of the challenges of spectral clustering \cite{Luxburg2007}. Choosing a scalar similarity function (e.g. the RBF kernel) may be overly restrictive and underutilize useful information. This is partly the cause of scale issues in spectral clustering \cite{Zelnik2004}. Our approach allows more complex similarity functions, such as the absolute vector difference.

We believe these results will be particularly useful for image segmentation, community discovery and entity resolution. These are all applications (a) with a large number of clusters and (b) where we have access to some labeled edges. With a large number of clusters, it is unlikely we have training samples from every class, let alone enough samples to train a multi-class supervised classifier. However, the small number of labeled edges will enable us to learn the typical cluster structure.

In this paper, we extend the planted partition model to general edge features. We also show how to partially recover a maximum likelihood estimator which is $\mathcal{O}(\log(n))$-close to the log likelihood of the true MLE by using an LP-rounding technique. Much of the analysis in planted partition models consider the probability of exactly recovering the partition. Depending on the cluster sizes and number of samples, this is often improbable. Our analysis addresses how good the result will be, regardless if it is exactly correct. Further, our theoretical results provide some insights on how to perform edge feature selection or, likewise, how to choose a similarity function for clustering. Experimental results show interesting clustering capabilities when leveraging edge feature vectors.

\section{Related Work}
Two areas of research are closely related to our work. Our graphical model is an extension of the stochastic block model from the mathematics and statistics literature. We also use some key results from correlation clustering in our algorithm and analysis.

\subsection{Stochastic Block Model}
The stochastic block model (SBM) was first studied by Holland et al. \yrcite{Holland1983} and Wang and Wong \yrcite{Wang1987} for understanding structure in networks. In its simplest form, every edge in the graph corresponds to a Bernoulli random variable, with probability depending on the two endpoints' clusters. In planted partition models\footnote{There is some inconsistency in the literature regarding the distinction between planted partition and stochastic block models. Occasionally the terms are used interchangeably} there are two Bernoulli probabilities $p$ and $q$ corresponding to if the endpoints are in the same or different clusters, respectively. These models are actually generalizations of the Erd\H{o}s-R\'enyi random graph, where $p=q$. Random graph models have a storied history and include famous studies such as the small-world experiment (popularized as ``six-degrees of separation'') by Milgram \yrcite{Milgram1967} and Zachary's Karate Club network \yrcite{Zachary1977}. For a more complete overview, we refer the interested reader to the review by Goldenberg et al. \yrcite{Goldenberg2010}.

More recently, planted partition models have gained popularity in the machine learning community for clustering. McSherry \yrcite{McSherry2001} and Condon \& Karp \yrcite{Condon2001} provided early spectral solutions to exactly recovering the correct partition, with probability depending on a subset of the parameters $p$, $q$, the number of samples $n$, the number of clusters $k$, and the smallest cluster size. Most results show recovering the partition is more difficult when $p$ and $q$ are close, $n$ is small, $k$ is large, and the smallest cluster size is small. Intuitively, if there are a high proportion of singleton clusters (i.e. ``dust''), mistaking at least one of them for noise is likely.

Some of the numerous alternative approaches to recovering the partition include variational EM \cite{Daudin2008, Airoldi2008, Park2010}, MCMC \cite{Park2010}, and variational Bayes EM \cite{Hofman2008, Aicher2013}. Some of these approaches may also be applicable to the model in this paper, though we found our approach simple to theoretically analyze.

The work most closely related to ours extends the stochastic block model edge weights to other parametric distributions. Motivated by observations that Bernoulli random variables often do not capture the degree complexity in social networks, Karrer \& Newman \yrcite{Karrer2011}, Mariadassou et al. \yrcite{Mariadassou2010} and Ball et al. \yrcite{Ball2011} each used Poisson distributed edge weights. This may also be a good choice because the Bernoulli degree distribution is asymptotically Poisson \cite{Xiaoran2007}. Aicher et al. considered an SBM with edge weights drawn from the exponential family distribution \yrcite{Aicher2013}. Like Thomas \& Blitzstein \yrcite{Thomas2011}, he also showed better results than thresholding to binary edges. Lastly, Balakrishnan et al. \yrcite{Balakrishnan2011} consider Normally distributed edge weights as a method of analyzing spectral clustering recovery with noise.

Other interesting extensions of the SBM include mixed membership (i.e.\ soft clustering) \cite{Airoldi2008}, hierarchical clustering \cite{Clauset2007, Balakrishnan2011} and cases where the number of clusters $k$ grows with the number of data points $n$ \cite{Rohe2011, Choi2012}.  Combining our ideas on general edge features with these interesting extensions should be possible.

\subsection{Correlation Clustering}
Correlation clustering was introduced by Bansal et al. \yrcite{Bansal2004} in the computer science and machine learning literature. Given a complete graph with $\pm 1$ edge weights, the problem is to find a clustering that agrees as much as possible with this graph. There are two `symmetric' approaches to solving the problem. \textsc{MinimizeDisagreements} aims to minimize the number of mistakes (i.e.\ $+1$ inter-cluster and $-1$ intra-cluster edges), while \textsc{MaximizeAgreements} aims to maximize the number of correctly classified edges (i.e.\ $-1$ inter-cluster and $+1$ intra-cluster edges). While the solutions are identical at optimality, the algorithms and approximations are different.

The original results by Bansal et al. \yrcite{Bansal2004} showed a constant factor approximation for \textsc{MinimizeDisagreements}. The current state-of-the-art for binary edges is a 3-approximation \cite{Ailon2008}, which Pan et al. \yrcite{Pan2015} recently parallelized to cluster one billion samples in 5 seconds. Ailon et al. \yrcite{Ailon2008} also showed a linear-time 5-approximation on weighted probability graphs and a 2-approximation on weighted probability graphs obeying the triangle inequality. Demaine et al. \yrcite{Demaine2006} showed an $\mathcal{O}(\log(n))$-approximation for arbitrarily weighted graphs using the results of Leighton \& Rao \yrcite{Leighton1999}. Solving \textsc{MinimizeDisagreements} is equivalent to APX-hard minimum multi-cut \cite{Demaine2006, Charikar2003}.

For \textsc{MaximizeAgreements}, the original results by Bansal et al. \yrcite{Bansal2004} showed a PTAS on binary graphs. State-of-the-art results for non-negative weighted graphs are a 0.7664-approximation by Charikar et al. \yrcite{Charikar2003} and similar 0.7666-approximation by Swamy \yrcite{Swamy2004}. Both results are based on Goemans and Williamson \yrcite{Goemans1995} using multiple random hyperplane projections.

Later, we will use correlation clustering to partially recover the maximum likelihood estimator of our planted partition model. Kollios et al. \yrcite{Kollios2013} consider a similar problem of using correlation clustering on probabilistic graphs, although their algorithm does not actually solve for the MLE.

\section{Problem Statement}
Consider observing an undirected graph $G = (V, E)$ with $n=|V|$ vertices. Let $\psi : \left\{1,\dotsc,n\right\} \rightarrow \left\{1, \dotsc, k\right\}$ be a partition of the $n$ vertices into $k$ classes. We use the notation $\psi_{ij}=1$ if nodes $i$ and $j$ belong to the same partition, and $\psi_{ij}=0$ else. Edges $e_{ij} \in \mathbb{R}^d$ are $d$-dimensional feature vectors. Note we say that graph $G$ is `observed,' though the edges $E$ may also be the result of a symmetric similarity function $s$, where $e_{ij} = s(v_i, v_j)$.

We assume a planted partition model, $e_{ij} \sim P(e | \psi_{ij})$. From now on, we will use the shorthand $P_0( \cdot) = P(\cdot | \psi_{ij} = 0)$ and $P_1( \cdot) = P(\cdot | \psi_{ij} = 1)$. In the conventional planted partition model, $P_0$ and $P_1$ are Bernoulli distributions with parameter $q$ and $p$, respectively. However, in this work we make no assumptions about the probability density functions $P_0$ and $P_1$. We will make the key assumption that all stochastic block models make -- that the edges are independent and identically distributed, conditioned on $\psi$. Note if the edges $E$ are generated by a similarity function then it is unlikely the edges are actually independent, but we proceed with this assumption regardless.

In most planted partition models, the goal is to either partially or exactly recover $\psi$ after observing $G$. We aim to find the most likely partition, and bound our performance in terms of the likelihood. There is a subtle distinction between the two goals. Even if the maximum likelihood estimator is consistent, the non-asymptotic MLE may be different than the true partition $\psi$.

\section{Approximating the Maximum Likelihood Estimator}
Let $\theta : \left\{1,\dotsc,n\right\} \rightarrow \left\{1, \dotsc, k\right\}$ be a partition under consideration. In exact recovery, our goal would be to find a $\theta$ such that $\theta = \psi$. However, our goal is to find a partition $\hat \theta$ which is close to the likelihood of the maximum likelihood estimator $\theta_{MLE}$. Using the edge independence assumption, the likelihood $L$ is
\begin{equation}
L(\theta) = \prod_{i < j} P(e_{ij} | \theta_{ij}) \mathbbm{1}(\theta \in \Theta)
\end{equation}
where $\Theta$ is the space of all disjoint partitions. 

The trick to finding an approximation $\hat \theta$ to the MLE $\theta_{MLE}$ is to reduce the problem to a correlation clustering instance. Consider forming a graph $G_O = (V, E_0)$ with binary edges defined by the sign of the log-odds $e_{0; ij} = sign\left(\log\left(P_1(e_{ij})/P_0(e_{ij})\right)\right)$. Let the cost of mislabeling each edge be the absolute log-odds $C_{ij} = \left|\log\left(P_1(e_{ij})/P_0(e_{ij})\right)\right|$. Then we can rewrite the log-likelihood $\ell$ as\footnote{We are playing fast and loose with the $\mathbbm{1}(\theta \in \Theta)$ terms here. $G_0$ is not required to be a valid partition and thus the $\mathbbm{1}(G_0 \in \Theta)$ term is not included in $\ell(G_0)$. However, $\theta$ is still required to be a valid partition.}
\begin{align}
\ell(\theta) &= \ell(G_0) - \sum_{\theta_{ij} \neq e_{0; ij}} \left|\log\left(\frac{P_1(e)}{P_0(e)}\right)\right| \nonumber \\
&= \ell(G_0) - \sum_{\theta_{ij} \neq e_{0; ij}} C_{ij} \label{eq:log-likelihood}
\end{align}

Maximizing $\ell(\theta)$ is equivalent to minimizing $\sum_{\theta_{ij} \neq e_{0; ij}} C_{ij}$, which is exactly \textsc{MinimizeDisagreements} where edges are labeled according to $E_0$ and have weighted costs $C\geq 0$. Intuitively, we consider the most likely graph $G_0$ (which is not a valid partition) and try to find the minimum number of weighted edge flips required to create a valid partition.

Unfortunately, we only have non-negativity bounds on the weights $C$. Thus we believe the only appropriate \textsc{MinimizeDisagreements} algorithm to solve Eq \ref{eq:log-likelihood} is the LP-rounding technique by Demaine et al. \yrcite{Demaine2006}.

\begin{theorem}
The above estimated clustering $\hat \theta$ is $c_1DIS\log(n)$-close to the log-likelihood of the true maximum likelihood estimator $\hat \theta_{MLE}$. This is an $\exp(-DIS(c_1\log(n)-1))$-approximation algorithm for the likelihood.
\end{theorem}

The constant $c_1 = 2+ 1/\log(n+1)$ is just slightly larger than 2. $DIS$ is a measure of disagreement between the graph $G_0$ and the optimal clustering, to be discussed shortly.

\begin{proof}
The results follow directly from Leighton \& Rao \yrcite{Leighton1999} and Demaine et al. \yrcite{Demaine2006}. Let $DIS$ be the optimal solution to \textsc{MinimizeDisagreements} on graph $G_0$ with weighs $C$. Then the log likelihood of the true MLE $\theta_{MLE}$ is
\begin{equation}
\ell(\theta_{MLE}) = \ell(G_0) - DIS
\end{equation}
Demaine et al. \yrcite{Demaine2006} showed an $c_1\log(n)$-approximation to \textsc{MinimizeDisagreements} on general weighted graphs. Thus the approximated MLE using this algorithm will yield
\begin{equation}
\ell(\hat \theta) \geq \ell(G_0) - c_1\log(n)DIS
\end{equation}
The approximation ratio result follows likewise.
\begin{align}
L(\hat \theta) &\geq L(G)\exp(-c_1\log(n)DIS) \nonumber \\
L(\theta_{MLE}) &= L(G)\exp(-DIS) \nonumber \\
\frac{L(\hat \theta)}{L(\theta_{MLE})} &\geq \exp(-DIS(c_1\log(n)-1)) \nonumber
\end{align}
\end{proof}

\subsection{Choosing Edge Features or Similarity Functions}
How to choose a similarity function remains a fundamental question in spectral clustering \cite{Luxburg2007}. A ``meaningful'' similarity function should have high similarity for samples belonging to the same cluster and low similarity for samples in different clusters, but how to judge that remains unclear. In practice, the radial basis function is commonly used and often provides favorable results. More precisely, we want to know which similarity functions make clustering easier and understand why they do. 

This same question applies when doing edge feature selection. We want to choose features which are most informative for clustering and ignore the others. We can provide a more scientific answer to these questions by analyzing the $DIS$ coefficient.

\begin{theorem}
Let $n_0$ and $n_1$ be the number of inter and intra-cluster edges in $\psi$, respectively. Then
\begin{align}
\mathbb{E}[DIS] &= -n_1D_{KL}(P_1 || P_0)\Big|_{P_1 \leq P_0} \nonumber \\
& \qquad \qquad \qquad \qquad \quad - n_0D_{KL}(P_0 || P_1)\Big|_{P_0 \leq P_1}
\end{align}
where we use the notation $D(\cdot || \cdot )\Big|_{S}$ to denote the divergence evaluated only over the closed set $S$.
\end{theorem}

\begin{proof}
\begin{align}
\mathbb{E}[DIS] &= (n_1 + n_2)\mathbb{E}_{\psi_{ij} \neq e_{0; i,j}}[C] \nonumber \\
&= n_1\mathbb{E}_{\psi_{ij}=1, e_{0;i,j}=0}[C] + n_2\mathbb{E}_{\psi_{ij}=1, e_{0;i,j}=0}[C] \nonumber \\
&= n_1\int_{P_1(e) \leq P_0(e)} P_1(e)\log\left(\frac{P_0(e)}{P_1(e)}\right)de \nonumber \\
& \quad + n_2\int_{P_0(e) \leq P_1(e)} P_0(e)\log\left(\frac{P_1(e)}{P_0(e)}\right)de \nonumber \\
&= -n_1D_{KL}(P_1 || P_0)\Big|_{P_1 \leq P_0} \nonumber \\
& \qquad \qquad \qquad \qquad \quad- n_0D_{KL}(P_0 || P_1)\Big|_{P_0 \leq P_1}\nonumber 
\end{align}
\end{proof}
Notice these restricted Kullback-Leibler divergences are always negative, and thus $\mathbb{E}[DIS] \geq 0$.

The intuition here is to choose edge features or similarity functions which are unlikely to create edges in the disagreement regions (i.e.\ edges which contribute to $DIS$). If $P_0$ and $P_1$ are completely divergent, then exactly recovering the partition is trivial because $G_0$ will be the set of disconnected cliques induced by $\psi$. Additionally, when mistakes are made, we want the KL divergence to be small (i.e.\ the mistake is not too `bad').

Along these lines, choosing higher dimensional edge features and similarity functions (e.g.\ the absolute vector difference instead of the Euclidean distance) makes clustering easier, by decreasing the disagreement region between $P_0$ and $P_1$. This confirms our earlier motivation that useful clustering information may be lost by only considering binary or scalar edge features and functions.

Considering only this approximation ratio when choosing a similarity function or edge features does not quite capture the complete picture. A trivial solution is select for $P_0 = P_1$ (a type of an Erd\H{o}s-R\'enyi random graph), which results in an approximation ratio of 1. Since every partition is equally likely in this scenario, finding an approximation to the MLE is trivial. However, exactly recovering $\psi$ is unlikely.  

\subsubsection{Sparsity}
In many situations it is advantageous to induce sparsity into the graph $G$. Spectral clustering employs this trick to cluster large graphs, by only considering the most similar nodes when performing eigen decompositions. In the proposed approach, sparsity will also reduce the number of variables and constraints in the LP used to maximize Eq \ref{eq:log-likelihood}.

By the previous analysis, we want to choose a similarity function or edge features which achieve the desired sparsity while maintaining a small $DIS$ coefficient. In the \textsc{MinimizeDisagreements} problem, sparse edges will have cost $C_{ij} = 0$. This occurs when $P_0(e_{ij}) = P_1(e_{ij})$. Intuitively, the best edges to sparsify are the ones which we do not have strong evidence for whether they should be labeled positive or negative. For these edges, the probabilities $P_0$ and $P_1$ will be close. Unlike spectral clustering, which only considers the most similar edges, this sparsification considers the most similar \emph{and} the most dissimilar edges.

\section{Experiments}
\begin{figure*}[p]%
\centering
\subfigure[True Clusters]{\includegraphics[width=0.25\linewidth]{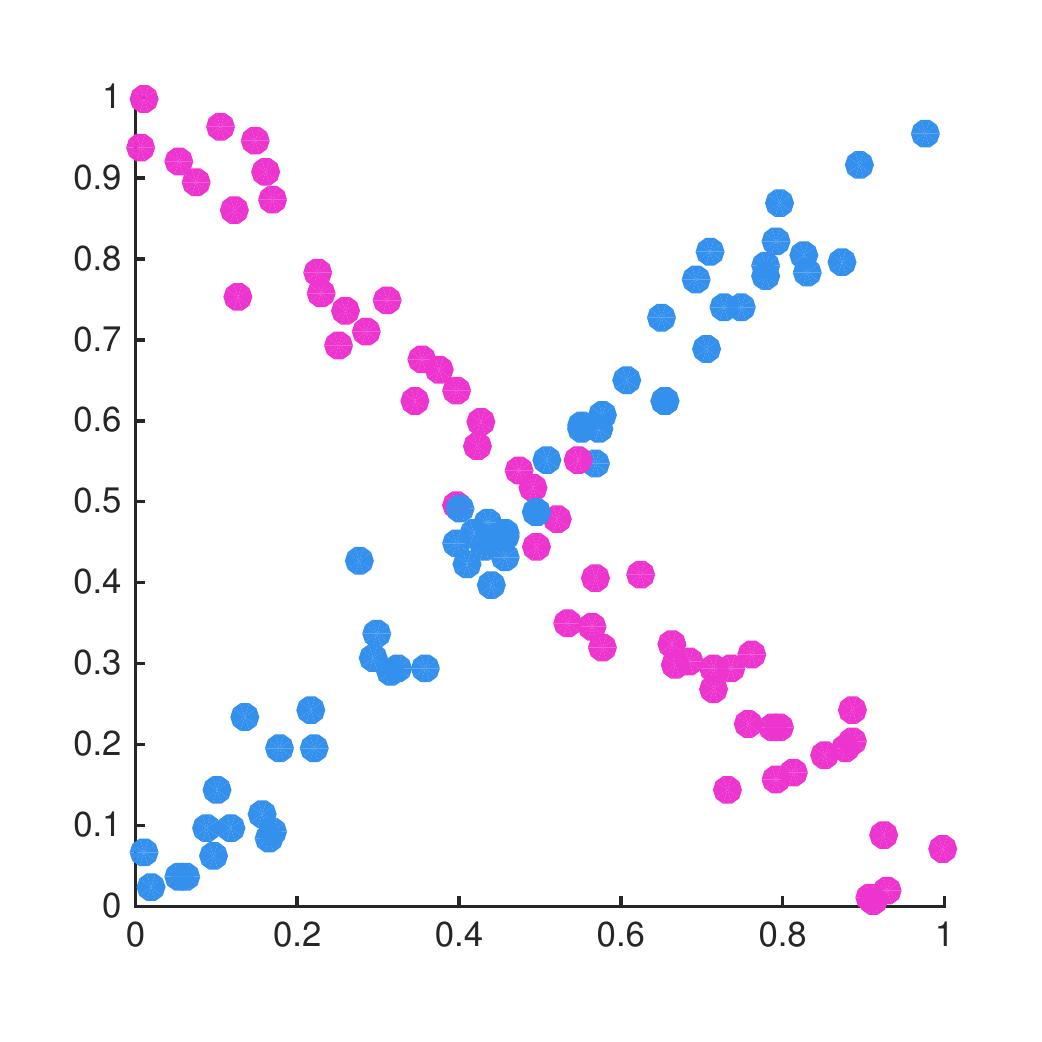}}\qquad
\subfigure[k-means]{\includegraphics[width=0.25\linewidth]{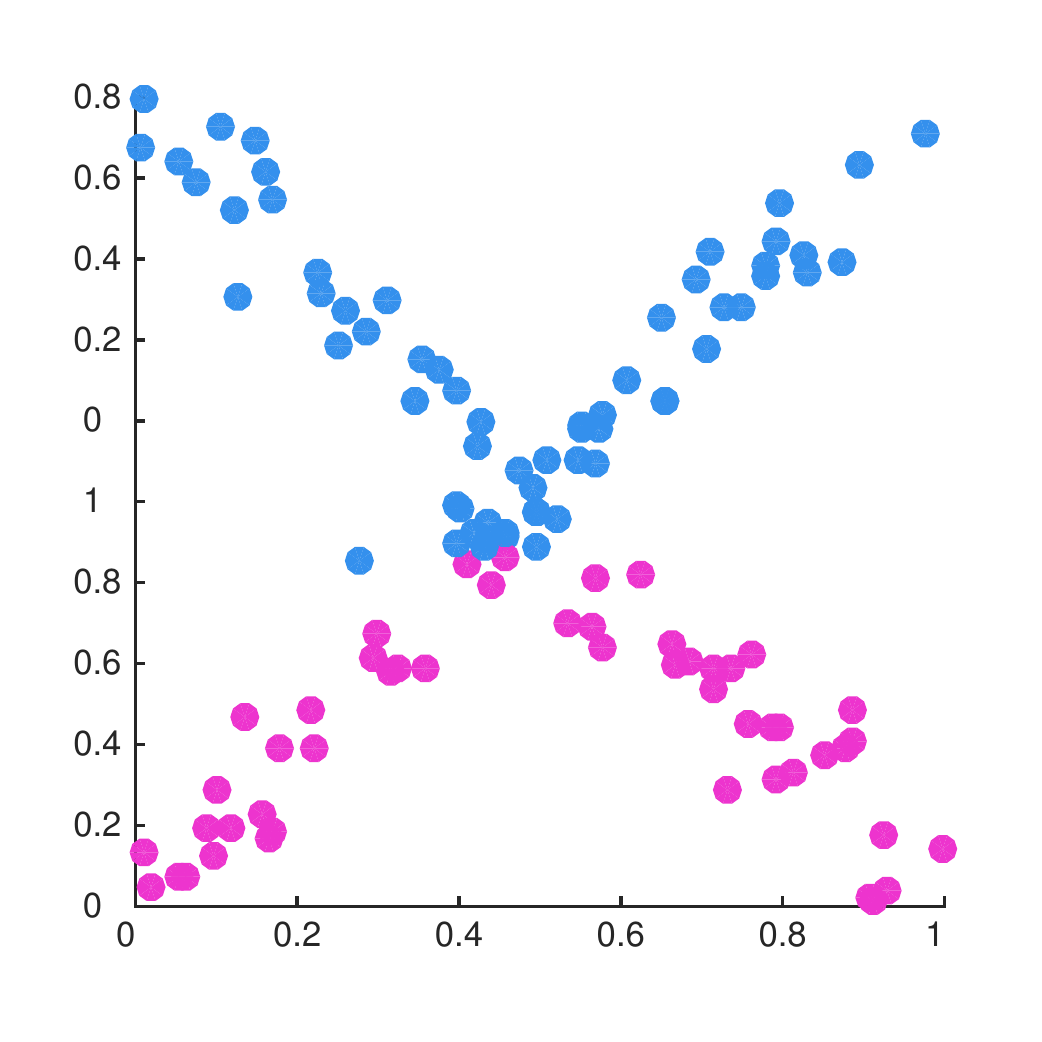}}\quad
\subfigure[Spectral]{\includegraphics[width=0.25\linewidth]{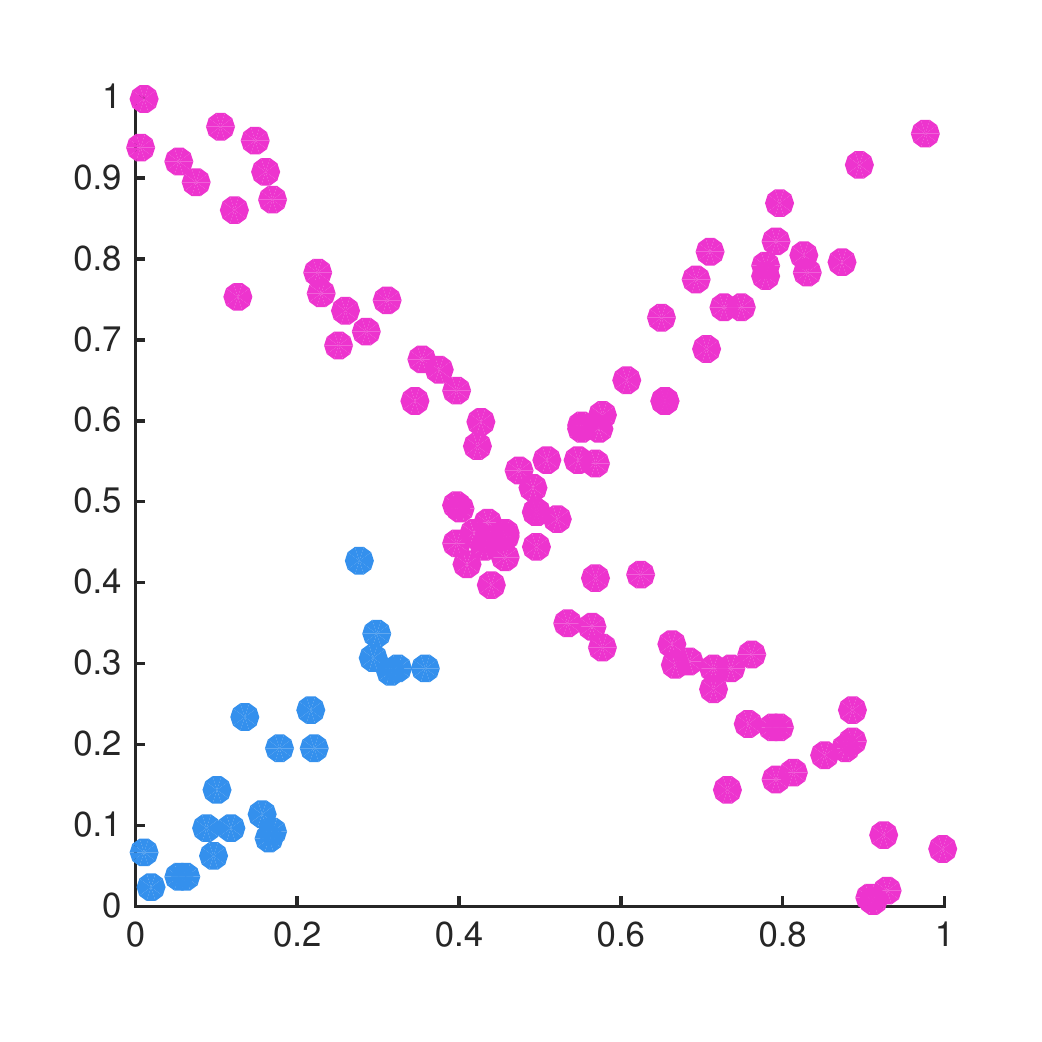}}\qquad\\
\subfigure[This Paper]{\includegraphics[width=0.25\linewidth]{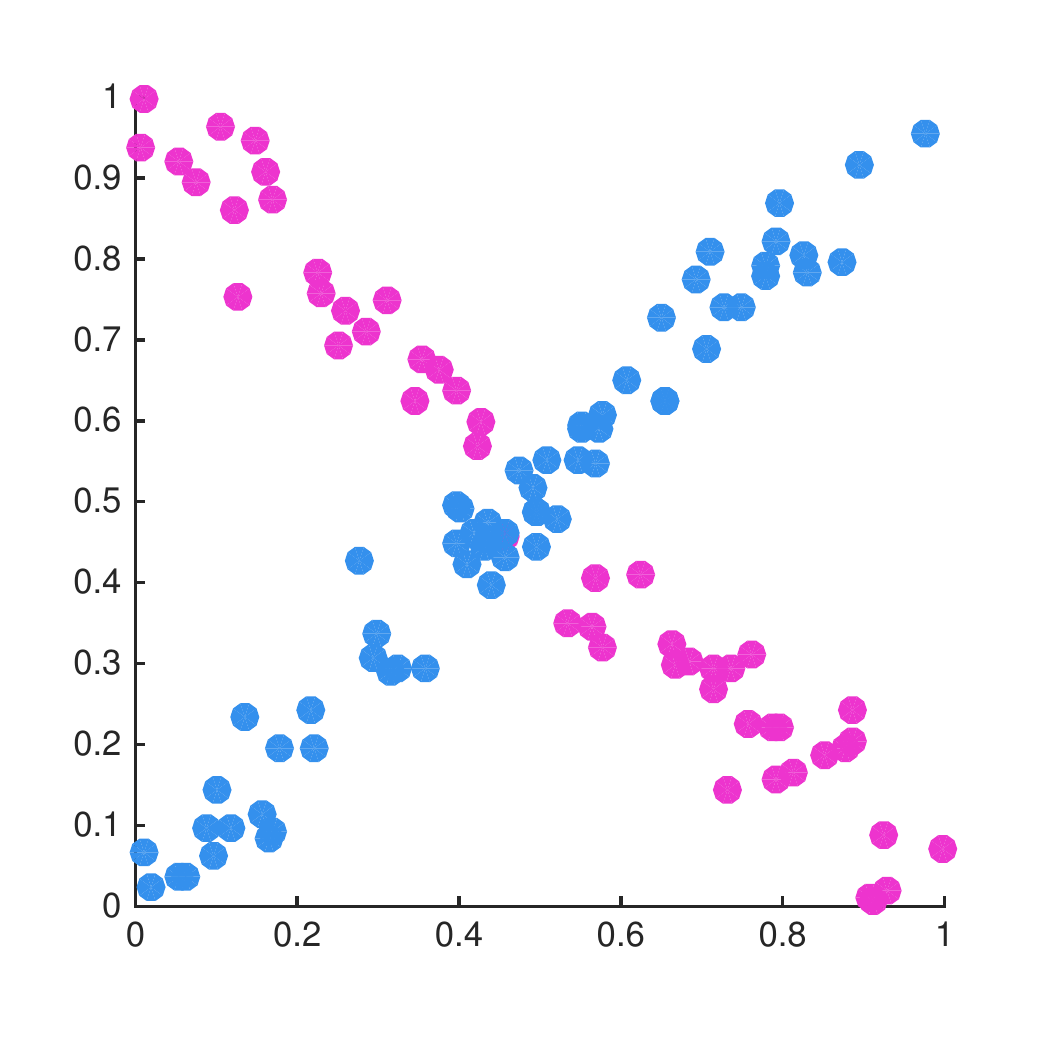}}\qquad
\subfigure[Cluster Structure $\hat P_1(e)$]{\includegraphics[width=0.25\linewidth]{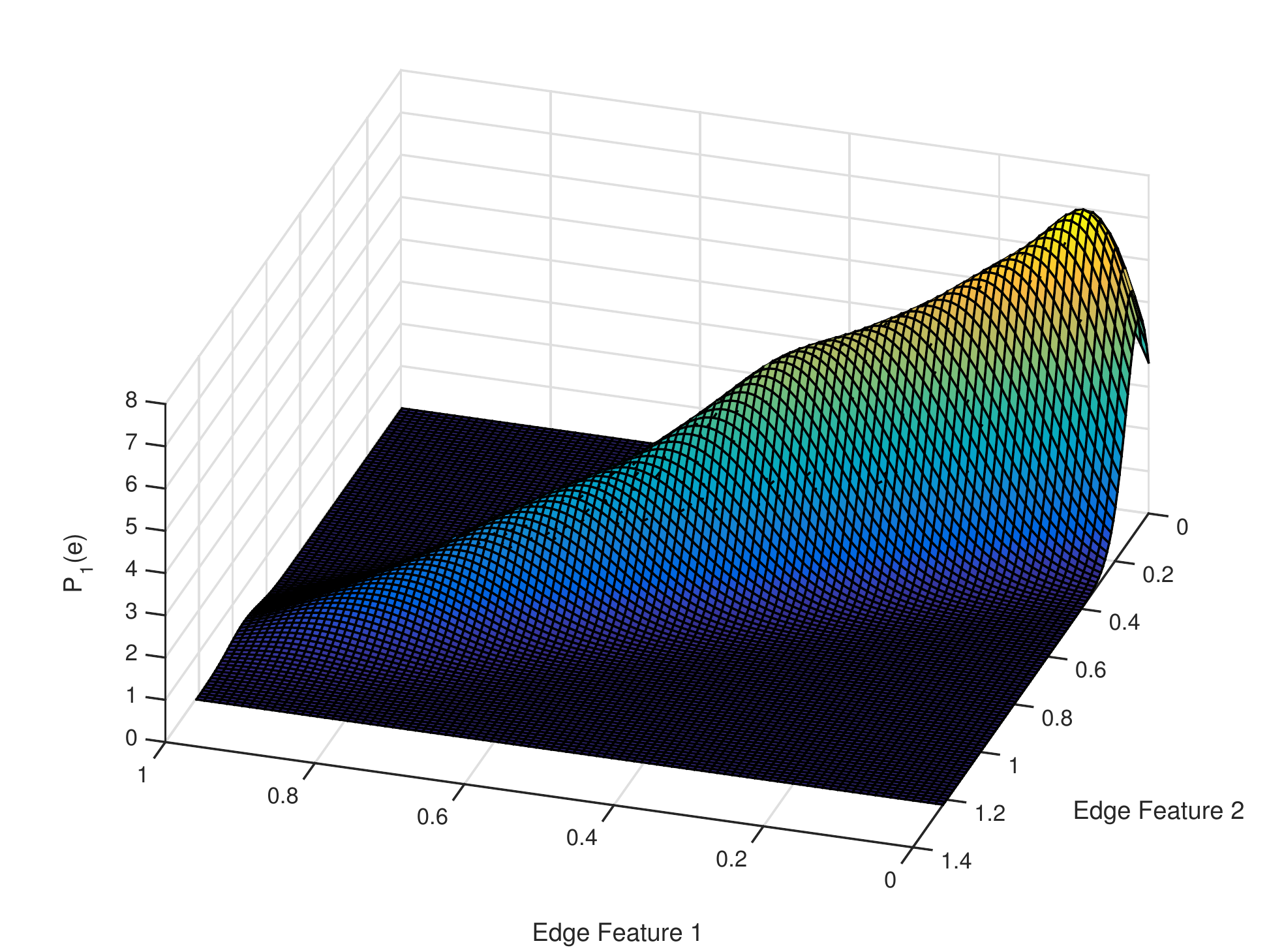}}\quad
\subfigure[Non-Cluster Structure $\hat P_0(e)$]{\includegraphics[width=0.25\linewidth]{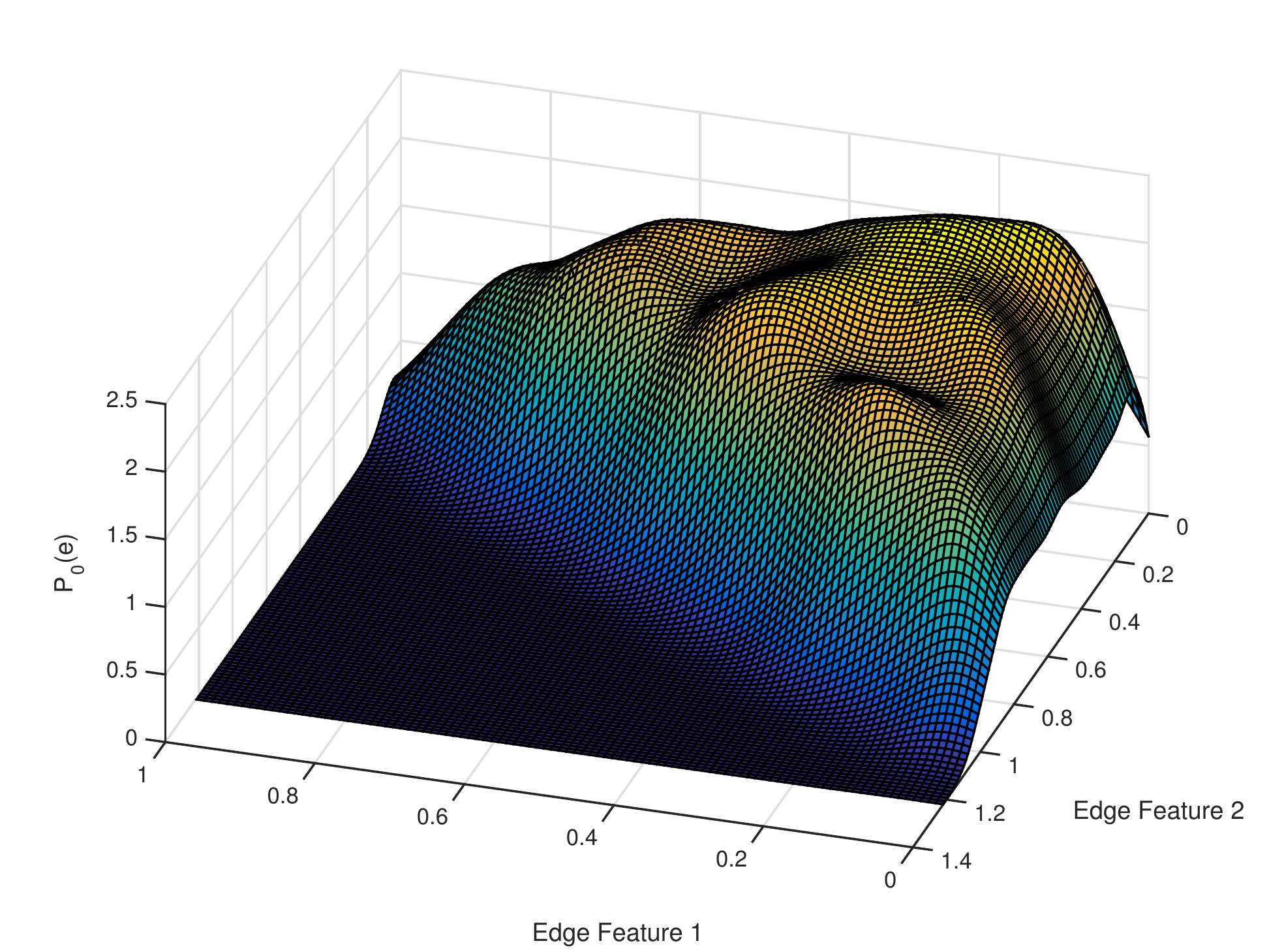}}\qquad\\
\subfigure[True Clusters]{\includegraphics[width=0.25\linewidth]{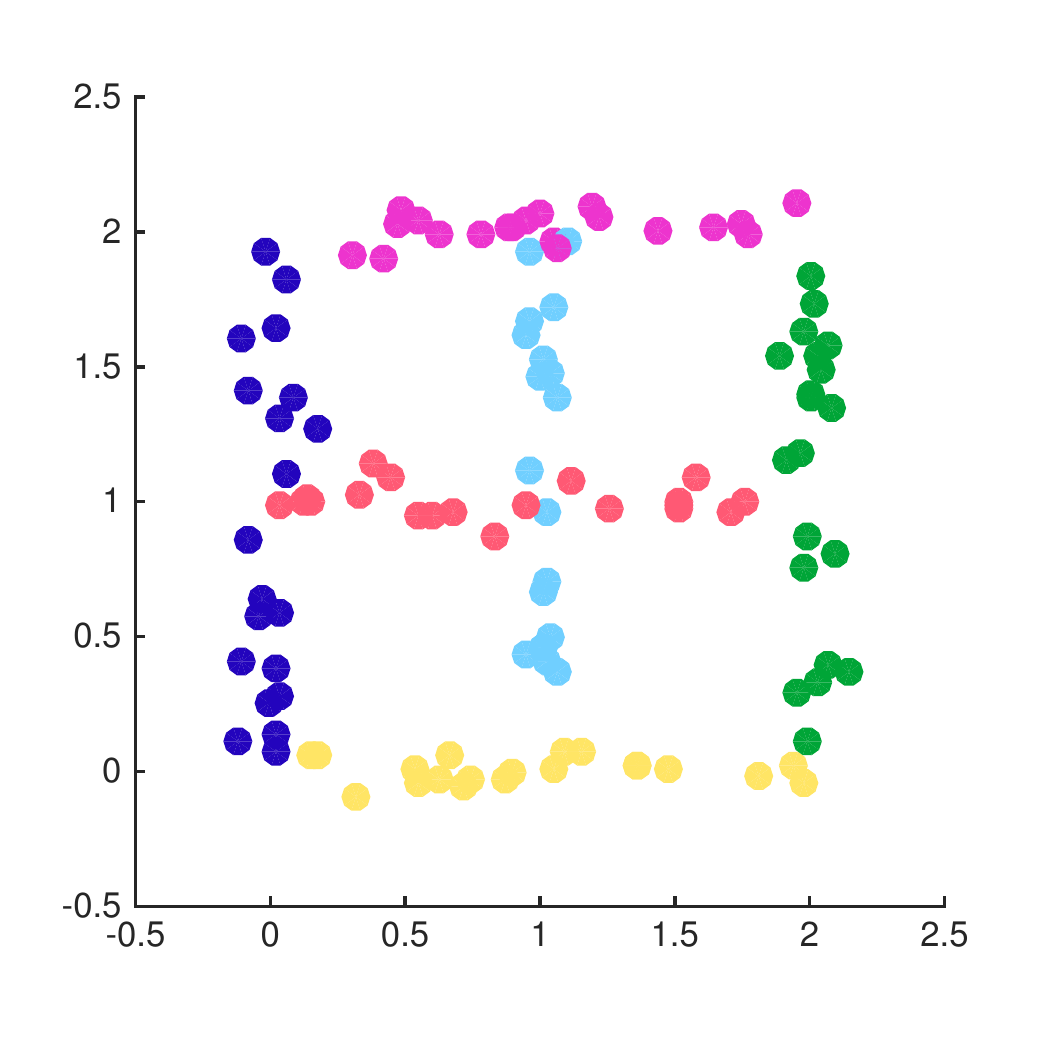}}\qquad
\subfigure[k-means]{\includegraphics[width=0.25\linewidth]{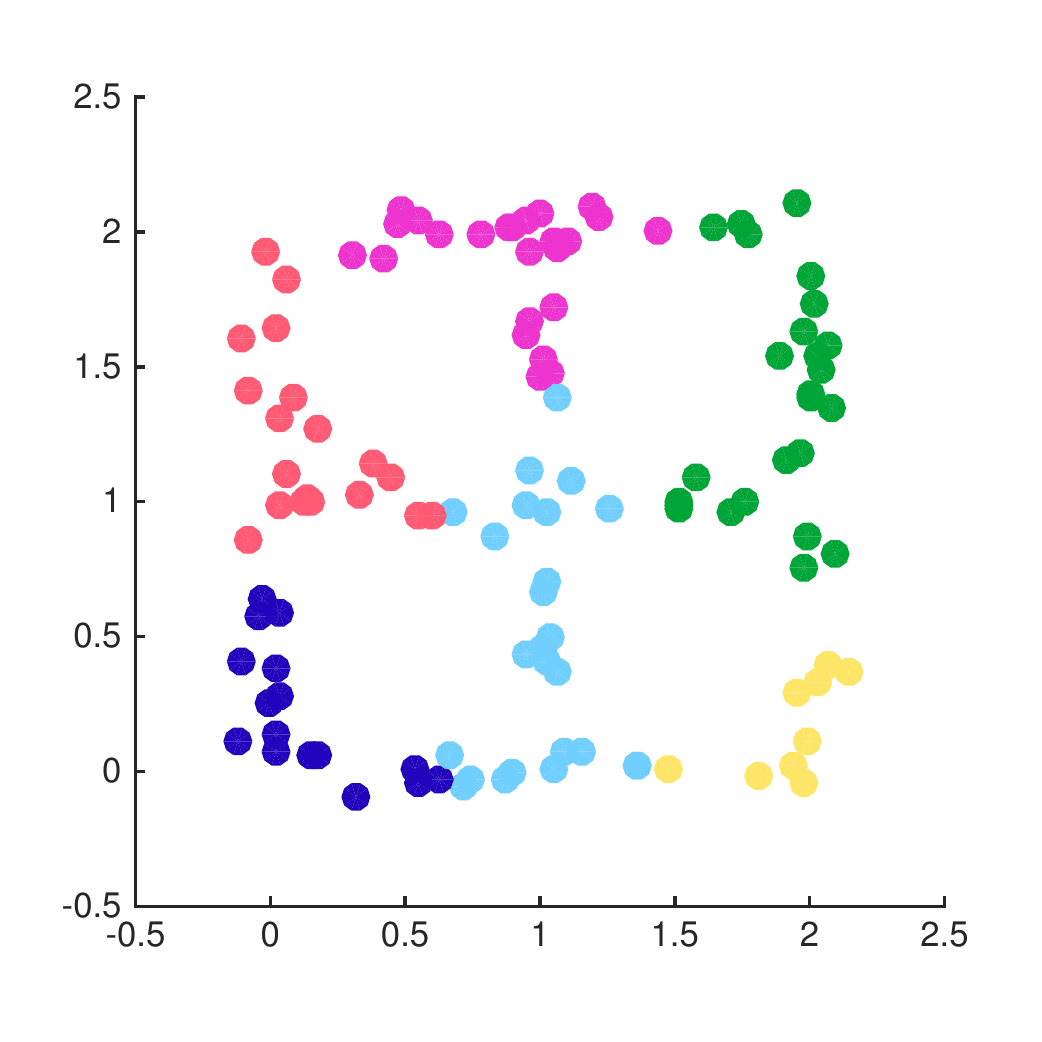}}\quad
\subfigure[Spectral]{\includegraphics[width=0.25\linewidth]{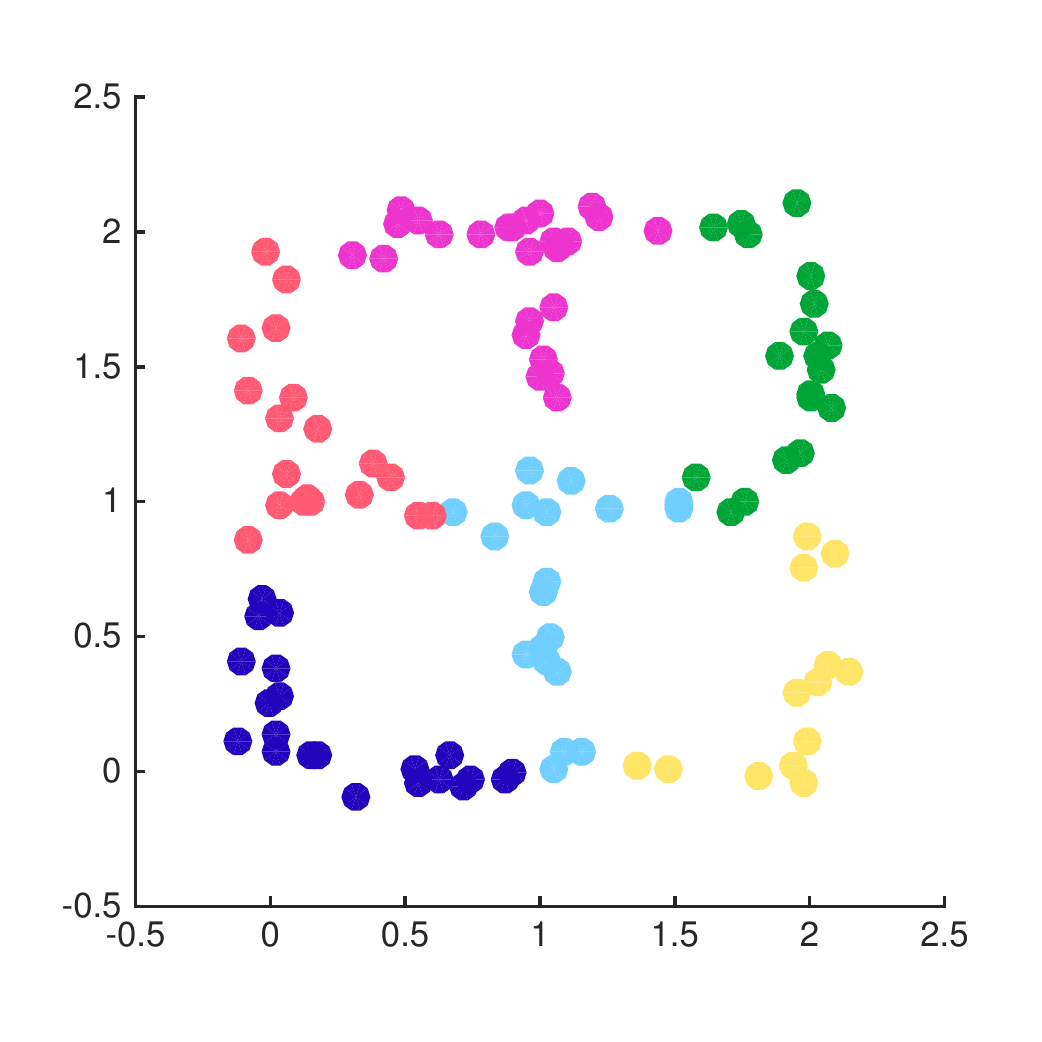}}\qquad\\
\subfigure[This Paper]{\includegraphics[width=0.25\linewidth]{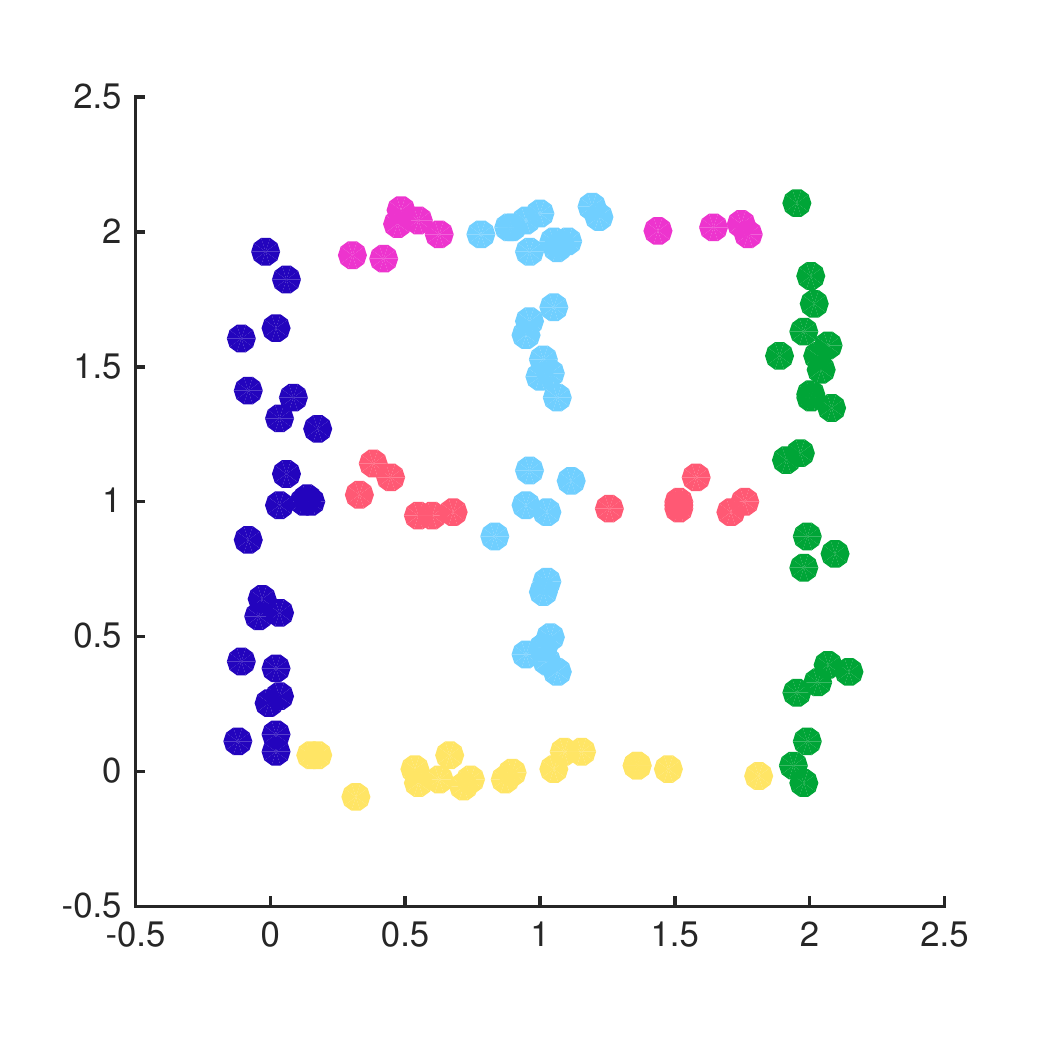}}\qquad
\subfigure[Cluster Structure $\hat P_1(e)$]{\includegraphics[width=0.25\linewidth]{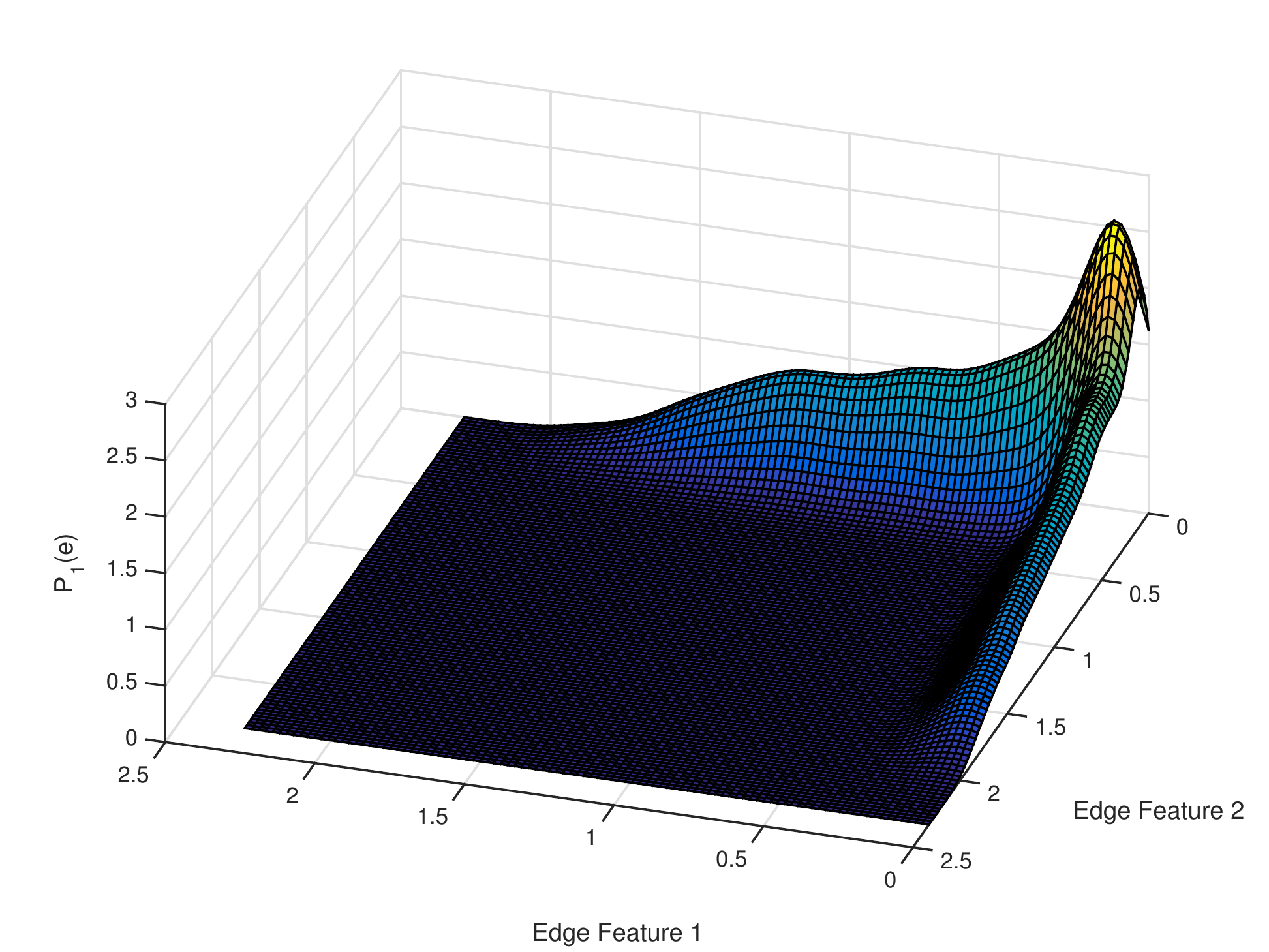}}\quad
\subfigure[Non-Cluster Structure $\hat P_0(e)$]{\includegraphics[width=0.25\linewidth]{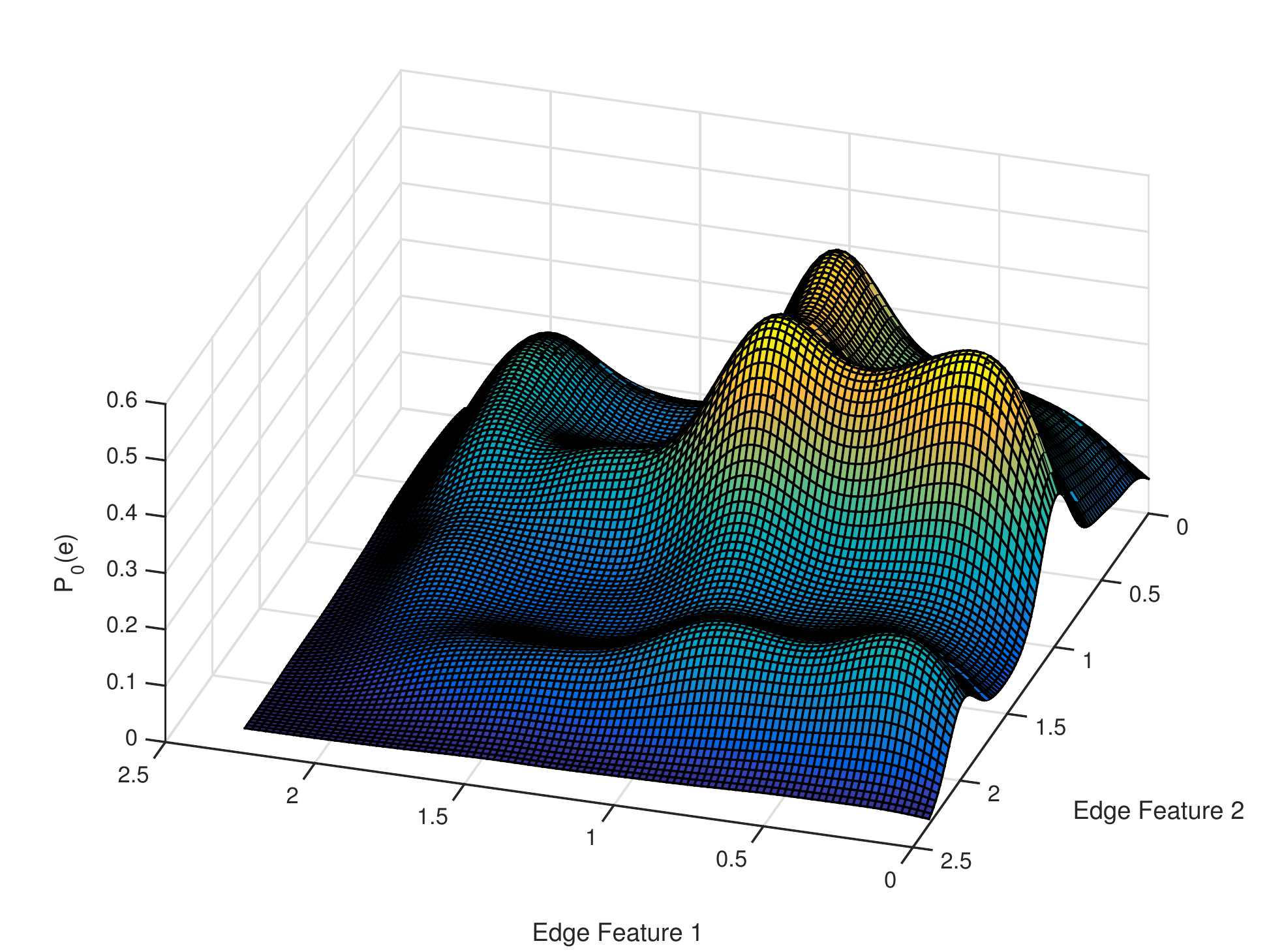}}
\caption{Results on synthetic 2D datasets. The true number of clusters $k$ is given as an input to k-means and spectral clustering, while our model consistently and naturally learns the correct number of clusters. The learned edge densities $P_0$ and $P_1$ are also shown.}
\label{fig:2D}
\end{figure*}

\begin{figure*}[t]%
\centering
\subfigure[True Clusters]{\includegraphics[width=0.4\linewidth]{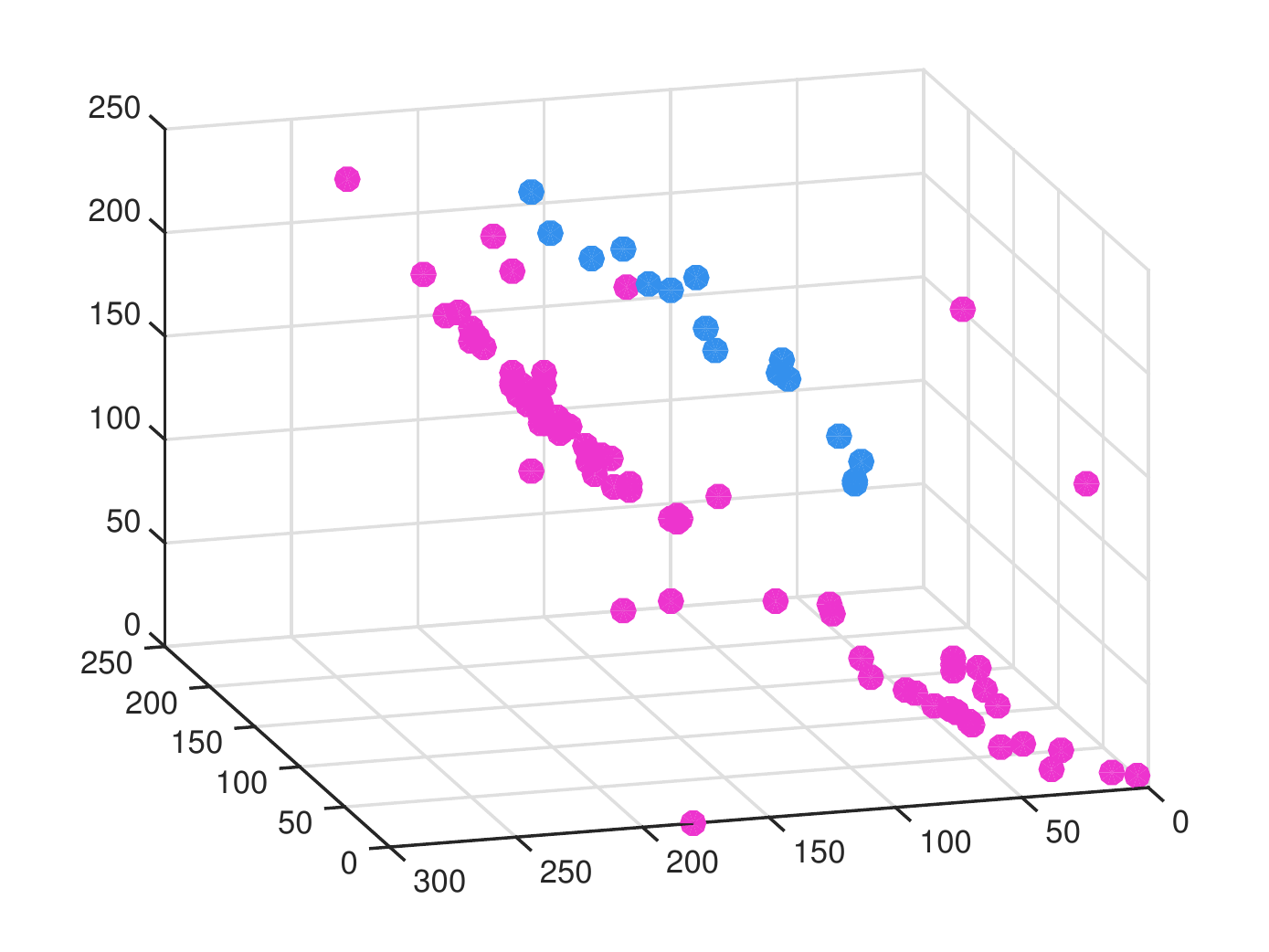}}\qquad
\subfigure[k-means]{\includegraphics[width=0.4\linewidth]{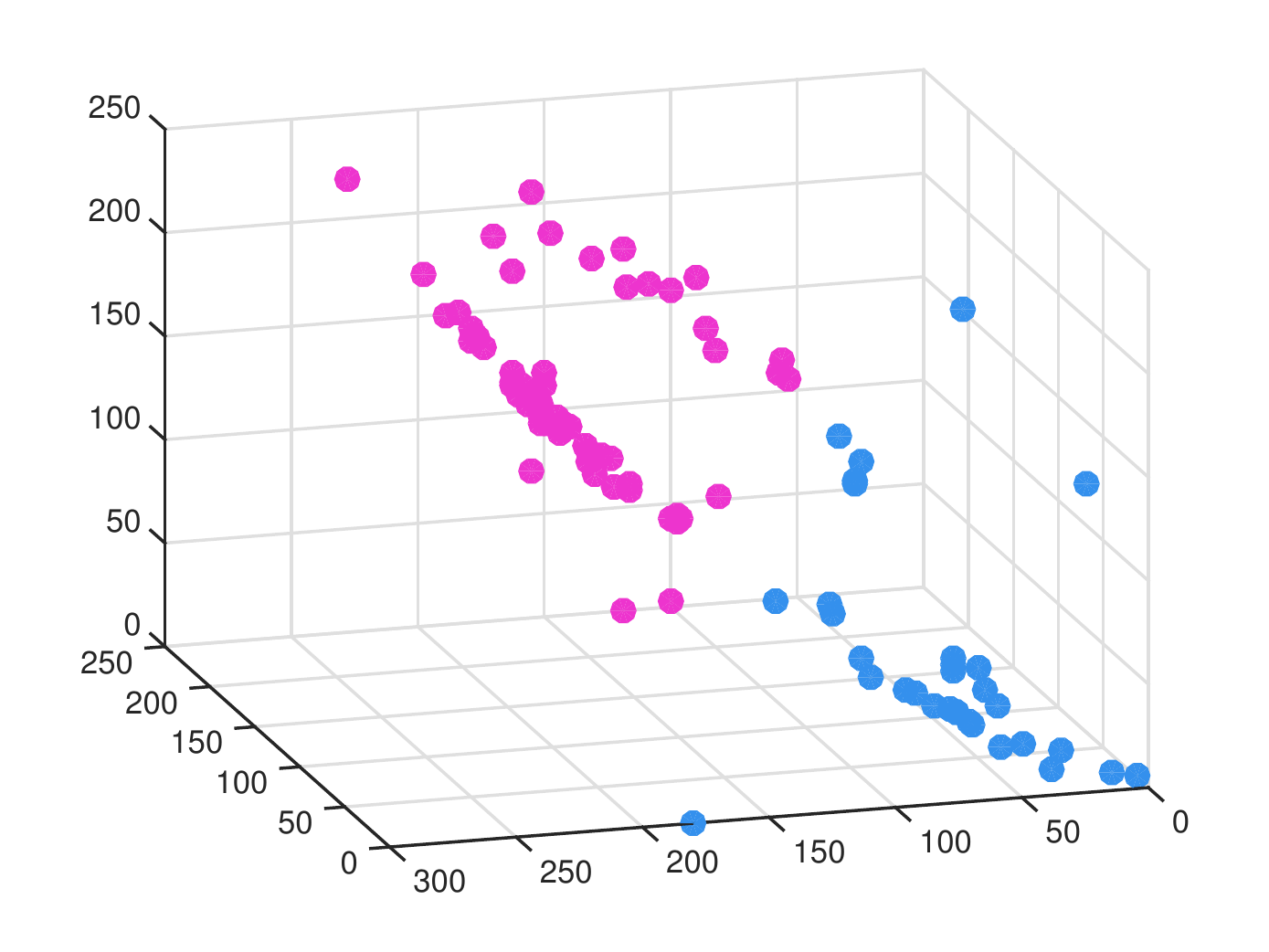}}\quad \\
\subfigure[Spectral]{\includegraphics[width=0.4\linewidth]{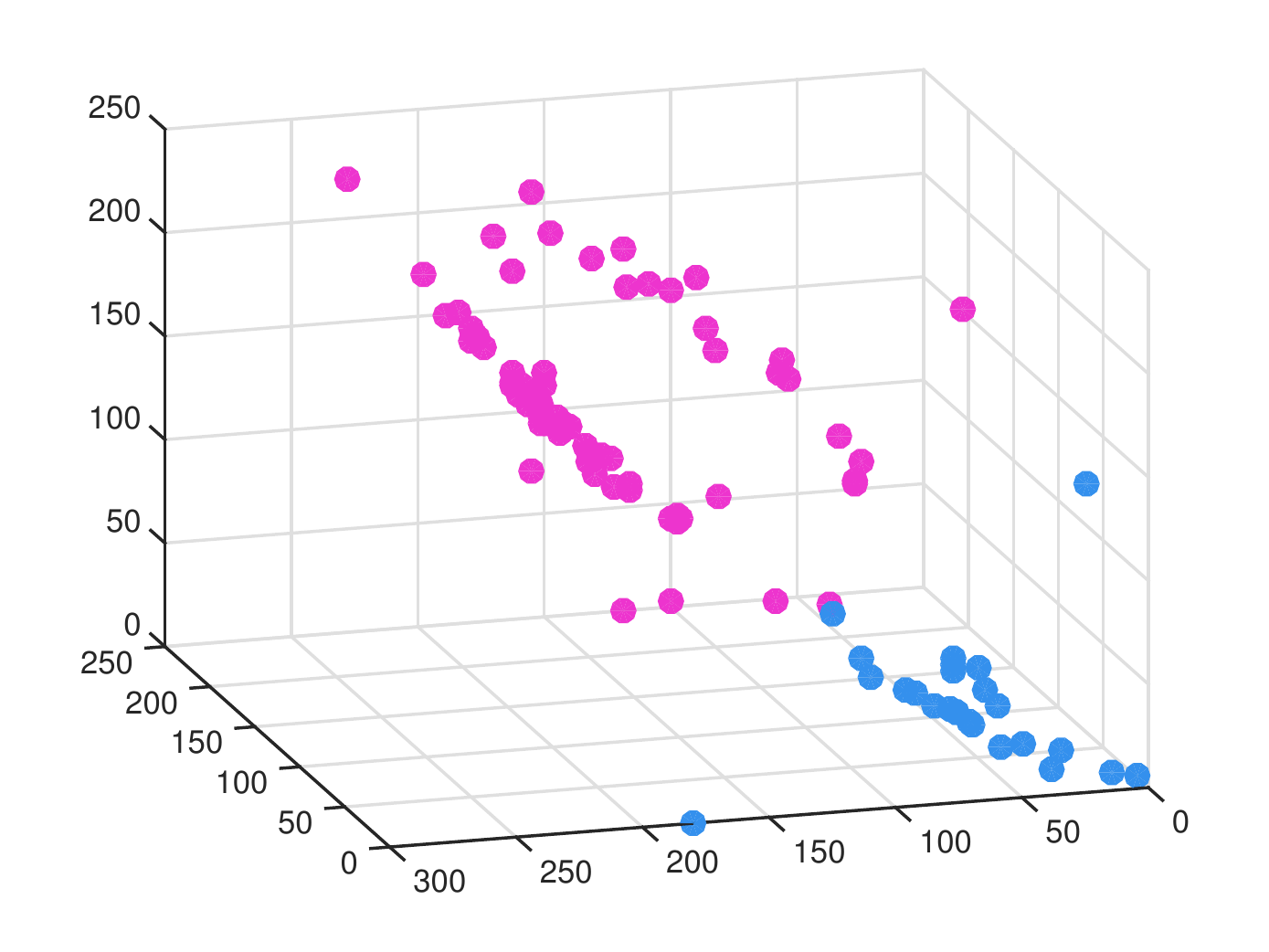}}\qquad
\subfigure[This Paper]{\includegraphics[width=0.4\linewidth]{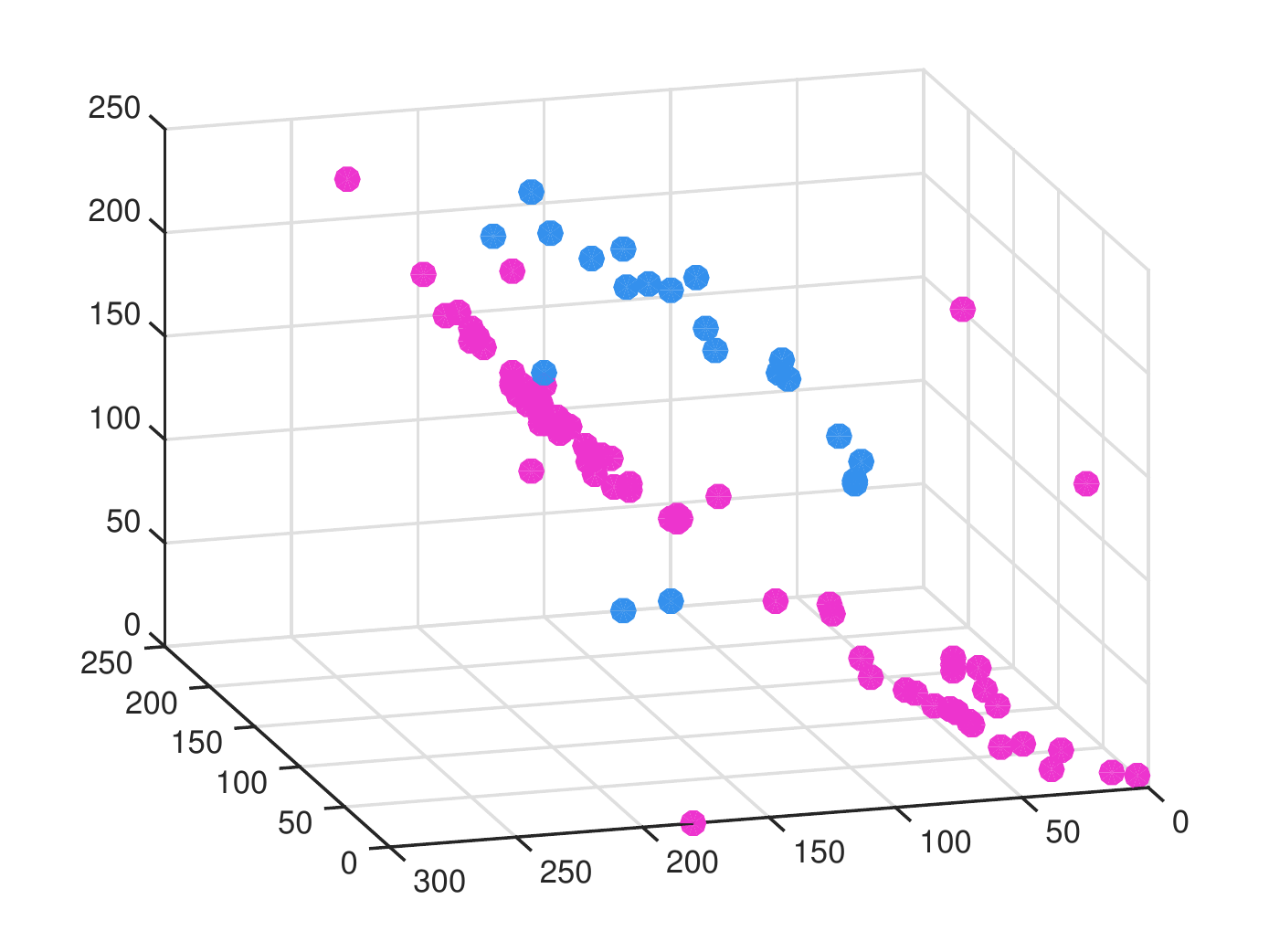}}\qquad
\caption{Results on the UCI Skin Segmentation dataset. Pink represents skin samples and blue represents non-skin samples. The axes correspond to RGB pixel values. The associated normalized mutual information scores are (b) 0.0042, (c) 0.1016 and (d) 0.6804.}
\label{fig:3D}
\end{figure*}

\begin{figure}[t]
\centering
        \includegraphics[width=\columnwidth]{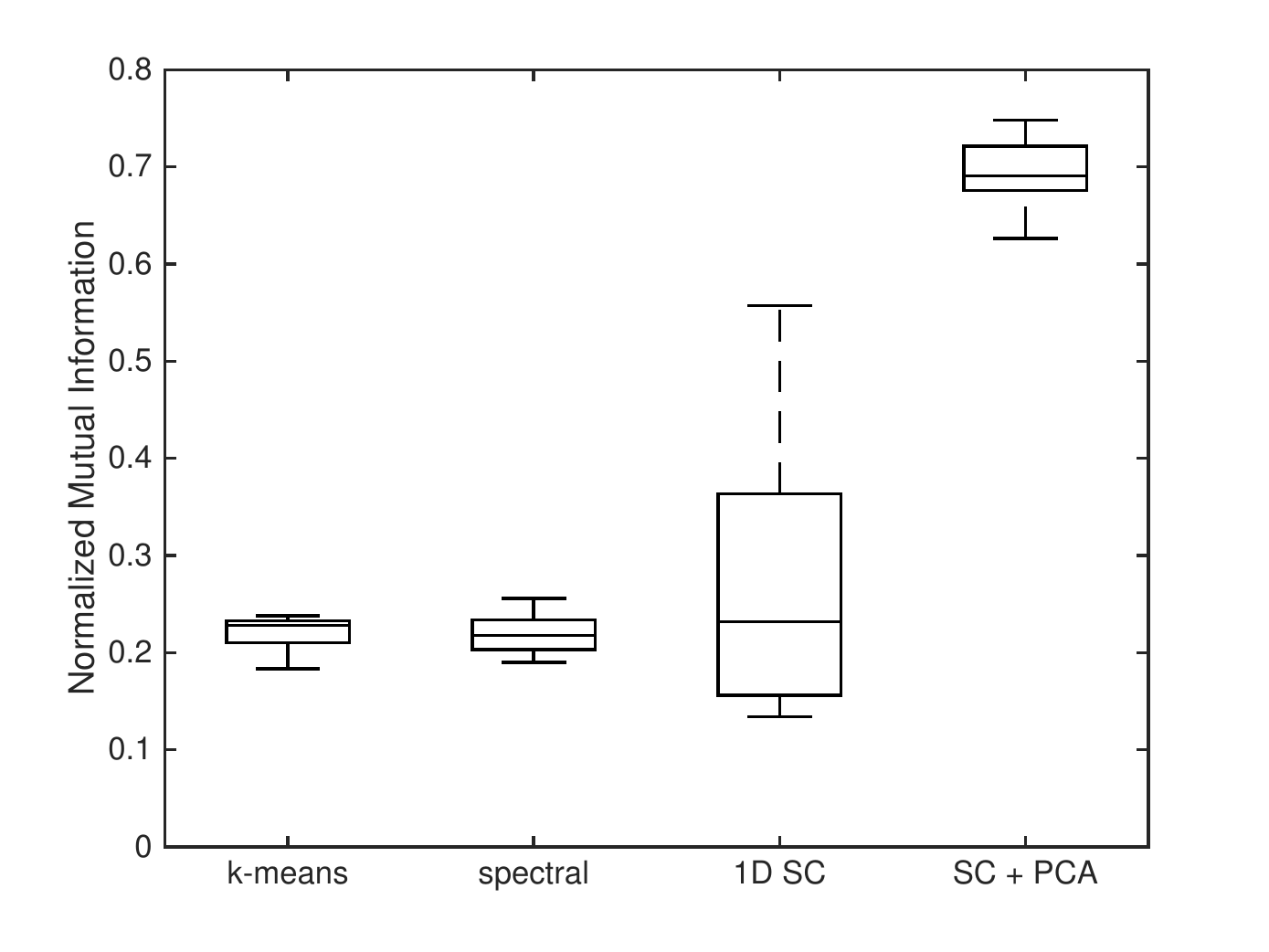}
        \cprotect\caption{Structured clustering with PCA (\verb!SC + PCA!) outperforms competitors on the 11 class, 48 dimensional UCI Sensorless Drive dataset. \verb!1D SC! is the same structured clustering model, except using the simpler Euclidean similarity function ($e_{ij} = ||v_i - v_j||_2$). Boxes correspond to the $25^{th}$ and $75^{th}$ percentile of 17 trials. Whiskers are the most extreme values.}
        \label{fig:motor}
\end{figure}

We will experimentally demonstrate the performance of the proposed model and algorithm on several synthetic and real world datasets. Specifically, we show studying edge features enables learning the \emph{structure} of clusters. When compared to k-means and spectral clustering, the planted partition model with general edge features can correctly cluster some rather interesting examples which are not attainable with scalar similarity functions. And more importantly, it seems to outperform existing methods on real world datasets.

In practice, it is unlikely we have access to $P_0$ and $P_1$. By assuming a prior parametric distribution, previous approaches have inferred these distributions while simultaneously learning the clustering. To remain as general as possible, we do not make any prior assumptions on $P_0$ and $P_1$. Our focus here is different. In datasets where the number of classes is very large or when new previously unseen classes are introduced, it is unlikely we can perform traditional supervised classification. This is especially true in the entity resolution and record linkage domains, where clusters correspond to millions or billions of entities (e.g. people, businesses, items) and new entities are frequently introduced. Statistical networks and image segmentation exhibit this property.

In these problems, we frequently have access to labeled \emph{pairs}. Human adjudicators are often able to say ``These two samples describe the same person'' or conversely ``these two samples describe different people.'' This is the information we use to learn the cluster structure. Thus we assume we have access to some labeled pairs in order to learn $P_0$ and $P_1$. Standard dimensionality reduction techniques can be employed to perform analysis in a reasonable space. We show this is still more powerful than using conventional scalar similarity functions. For all of our experiments, we use kernel density estimation to estimate $P_0$ and $P_1$. To improve the dimensional scalability, we could also perform a single estimation of $P_1/P_0$ using direct density ratio estimation \cite{Kanamori2009}

To compare performance we evaluate against k-means and spectral clustering \cite{Shi2000}. Per the recommendations of Luxburg \yrcite{Luxburg2007}, we use the Gaussian similarity function, a mutual $k$-nearest neighbor graph where $k=20$ and the random walk graph Laplacian. Unless otherwise noted, the edge features used for our method are from the absolute vector difference function $e_{ij} = |x=v_i - v_j|$ and thus not independent. However, the results indicate that it may be an acceptable assumption.

Lastly, our model consistently and naturally learns the correct number of clusters $k$. We found it occasionally labeled outlier samples as singleton clusters, though this would have a very small impact on the normalized mutual information score. For k-means and spectral clustering we do provide $k$ as an input. There are certainly methods of estimating $k$ for these competitors (e.g. analyzing the spectral gap), although they are not intrinsic to the methods.

\subsection{Results on Synthetic Data}
We consider the two interesting synthetic examples shown in Figure \ref{fig:2D}. Traditional clustering algorithms such as k-means and spectral clustering are unable to correctly label these examples because the clusters occasionally cross each other. Our method is able to capture the unidirectional cluster structure, and thus correctly label the samples. This is not an occasional event, in fact we have yet to see our method fail on these examples.

For all the synthetic experiments, we estimated $P_0$ and $P_1$ using 5,000 labeled pairs and clustered 100 hold-out samples. We use the absolute vector difference as our similarity function, which is able to capture the distance \emph{and} direction, unlike the Gaussian similarity function. There may be other excellent choices for similarity function, this is the only one we have tried so far.

We have achieved comparable results on the classic Gaussian, two moons, concentric circles and swiss roll examples. There was little distinction between our method and spectral clustering on these problems, so they were omitted from this paper.

\subsection{Results on Real World Data}
The first real world data we consider is the UCI Skin Segmentation dataset\footnote{https://archive.ics.uci.edu/ml/datasets/Skin+Segmentation}, shown in Figure \ref{fig:3D}. Samples are RGB values and labeled according to whether they are skin or non-skin image pixels. Again, we estimated $P_0$ and $P_1$ using 5,000 labeled pairs and clustered 100 hold-out samples, and use the absolute vector difference similarity function.

Visually, this seems much easier than the previous synthetic examples. However, k-means and spectral clustering are still unsuccessful due to the data scale issue introduced by the oblong cluster nature. Feature whitening did not help the competitors, though we believe some extensions to the standard spectral clustering may be able to handle this type of data \cite{Zelnik2004}.

The second realistic example we consider is the UCI Sensorless Drive Diagnosis dataset\footnote{https://archive.ics.uci.edu/ml/datasets/Dataset+for\\+Sensorless+Drive+Diagnosis}. Features are derived from current and frequency measurements in defective electric motors, including the statistical mean, standard deviation, skewness and kurtosis of intrinsic mode function subsequences. In total, there are 48 features and 11 classes.

We repeat the same previous procedures, except we additionally perform principal component analysis on the training and hold-out edge features prior to estimating $P_0$ and $P_1$ (\verb!SC + PCA!). We also consider one dimensional features using the Euclidean distance similarity function (\verb!1D SC!). The results from 17 trials are shown in Figure \ref{fig:motor}. The strong performance on the PCA reduced edge features leads us to believe that even if the original vertices have high dimensional structure, the distinguishing edge features in clusters have a lower dimensional representation.

\section{Conclusions}
Overall, incorporating multivariate edge features and more powerful similarity functions improves performance in all the experiments we have conducted. And even when the edges are clearly not independently generated, our structured clustering model still outperforms competitors.

The key insight from our approach is that multidimensional edge features can be used to effectively learn structure in clusters. Relationships in real world data are more complex than a simple scalar similarity function, and our methods can benefit from capturing that additional complexity. Then we can use the learned cluster structure to both determine the correct number of clusters and to handle situations where we are given new, previously unseen clusters, by assuming similar structure. 

Applications which may especially benefit from structured clustering usually (a) have some labeled edges to learn $P_0$ and $P_1$ and (b) have a large number of clusters which make training a supervised classifier impractical. For example, in community detection and entity resolution, we have many examples of communities or entities to learn $P_0$ and $P_1$, though we certainly do not have examples of every community and entity to perform classification. Intuitively, we expect communities and entities to exhibit some common behavior, and we can leverage this structure while clustering. In image segmentation we usually have many images with human labeled segments, but the segments (i.e.\ classes) in new images are likely of a different object. However, it is not unreasonable to assume the \emph{structure} of these new segments is similar to the previously seen segments. 

We used the approximation algorithm by Demaine et al. \yrcite{Demaine2006} for \textsc{MinimizeDisagreements}. The solution for this sub-problem is not the main focus of this paper, and unfortunately this particular algorithm requires solving a large linear program which limited the scalability of our experiments. 
Pan et al. recently clustered 1 billion samples in 5 seconds using a parallelizable, linear time algorithm for \textsc{MinimizeDisagreements}, but only with edge weight restrictions \yrcite{Pan2015}.


In future work, we intend to provide spectral solutions to this same problem, which we believe will help address scalability and provide better theoretical insights into the exact recoverability of the partition. A major goal of analyzing exact recoverability is understanding how to design smart, sparse and simple similarity functions. 

Other interesting extensions include applying the same method to stochastic block models, which would require estimating a separate $P_0$ and $P_1$ for every pair of blocks. In record linkage problems the same technique could be used to cluster vertices with different feature types. For example, clustering \emph{across} multiple social networks is of particular interest for advertising and law enforcement.

We have independently provided similar analysis for \textsc{MaximizeAgreements} by extending the results of Swamy \yrcite{Swamy2004} and Charikar et al. \yrcite{Charikar2003} to graphs with negative edge weights, though the theoretical and experimental results were not as convincing as \textsc{MinimizeDisagreements}, so thus omitted from this paper.

\bibliography{mybib}
\bibliographystyle{icml2016}

\end{document}